\documentclass[letterpaper]{article} 
\usepackage{aaai2026}  
\usepackage{times}  
\usepackage{helvet}  
\usepackage{courier}  
\usepackage[hyphens]{url}  
\usepackage{graphicx} 
\usepackage{natbib}  
\usepackage{caption} 
\usepackage{algorithm}
\usepackage{algorithmic}

\usepackage{amsmath}
\usepackage{amsthm}
\usepackage{subcaption}
\usepackage{caption}
\usepackage{thmtools, thm-restate}
\usepackage{booktabs}  
\usepackage{multirow}
\usepackage{xcolor,colortbl}
\usepackage{amsfonts}

\usepackage{newfloat}
\usepackage{listings}

\DeclareCaptionStyle{ruled}{labelfont=normalfont,labelsep=colon,strut=off} 
\floatstyle{ruled}
\newfloat{listing}{tb}{lst}{}
\floatname{listing}{Listing}
\newtheorem{theorem}{Theorem}[section]
\newtheorem{proposition}[theorem]{Proposition}
\newtheorem{lemma}[theorem]{Lemma}

\newtheorem{definition}[theorem]{Definition}

\newtheorem{remark}[theorem]{Remark}

\def\argmax{\mathop{\rm arg\, max}}

\newcommand{\LineComment}[1]{\STATE \texttt{\textcolor{blue}{$//$ #1 }}}

\title{Quantum Lipschitz Bandits}
\author
{
Bongsoo Yi\textsuperscript{\rm 1}, 
Yue Kang\textsuperscript{\rm 2}, 
Yao Li\textsuperscript{\rm 1} 
}
\affiliations{
    \textsuperscript{\rm 1}
    University of North Carolina at Chapel Hill, Chapel Hill, NC, USA\\
    \textsuperscript{\rm 2}
    Microsoft, Redmond, WA, USA\\
    \{bongsoo, yaoli\}@unc.edu, yuekang@microsoft.com
}

\begin{document}

\maketitle

\begin{abstract}
The Lipschitz bandit is a key variant of stochastic bandit problems where the expected reward function satisfies a Lipschitz condition with respect to an arm metric space. With its wide-ranging practical applications, various Lipschitz bandit algorithms have been developed, achieving the cumulative regret lower bound of order $\tilde O(T^{(d_z+1)/(d_z+2)})$\footnote{$\tilde{O}(\cdot)$ suppresses polylogarithmic factors, and $d_z$ is the zooming dimension defined in Section~\ref{sec:prelim}.} over time horizon $T$. Motivated by recent advancements in quantum computing and the demonstrated success of quantum Monte Carlo in simpler bandit settings, we introduce the first quantum Lipschitz bandit algorithms to address the challenges of continuous action spaces and non-linear reward functions. Specifically, we first leverage the elimination-based framework to propose an efficient quantum Lipschitz bandit algorithm named \text{Q-LAE}. Next, we present novel modifications to the classical Zooming algorithm~\cite{kleinberg2008multi}, which results in a simple quantum Lipschitz bandit method, \text{Q-Zooming}. Both algorithms exploit the computational power of quantum methods to achieve an improved regret bound of $\tilde O(T^{d_z/(d_z+1)})$. Comprehensive experiments further validate our improved theoretical findings, demonstrating superior empirical performance compared to existing Lipschitz bandit methods.
\end{abstract}

\section{Introduction}

The multi-armed bandit~\cite{slivkins2019introduction} is a fundamental and versatile framework for sequential decision-making problems, with applications in online recommendation~\cite{li2010contextual}, prompt engineering~\cite{lin2023use}, clinical trials~\cite{villar2015multi} and hyperparameter tuning~\cite{ding2022syndicated}. An important extension of this framework is the Lipschitz bandit~\cite{kleinberg2019bandits}, which addresses bandit problems with a continuous and infinite set of arms defined within a known metric space. The expected reward function in this setting satisfies a Lipschitz condition, ensuring that similar actions yield similar rewards. By leveraging the structure of the metric space, the Lipschitz bandit provides an effective and practical framework for solving complex problems with large or infinite arm sets, making it broadly applicable to real-world scenarios~\cite{mao2018contextual, feng2023lipschitz, tacs2021efficient}. After extensive studies on this intriguing problem~\cite{kleinberg2019bandits,slivkins2011contextual}, it is well-established that the cumulative regret lower bound for any Lipschitz bandit algorithm is $\tilde O\left(T^{{(d_z+1)}/{(d_z+2)}}\right)$, where $T$ is the time horizon and $d_z$ is the zooming dimension. While state-of-the-art algorithms achieve this optimal bound, the question remains whether more advanced techniques can further enhance performance. In this context, quantum computation~\cite{divincenzo1995quantum, biamonte2017quantum} presents an exciting opportunity to accelerate algorithmic performance and unlock new possibilities in solving Lipschitz bandit problems.

Quantum computation has been successfully integrated into various bandit frameworks, including multi-armed bandits~\cite{wan2023quantum}, kernelized bandits~\cite{dai2024quantum, hikima2024quantum}, and heavy-tailed bandits~\cite{wu2023quantum}, where it has demonstrated significant improvements in regret performance. In these quantum bandit settings, rather than receiving immediate reward feedback, we interact with quantum oracles that encode the reward distribution for each chosen arm. The Quantum Monte Carlo (QMC) method~\cite{montanaro2015quantum} is then utilized to efficiently estimate the expected rewards, requiring fewer queries than classical approaches. 
Despite these advancements, the application of quantum computing to Lipschitz bandits remains unexplored. Although Lipschitz bandits may appear as a natural extension of multi-armed bandits, they introduce fundamentally greater analytical challenges.
Unlike the discrete and finite action space in classical bandits, Lipschitz bandits operate over a continuous and uncountably infinite arm space, with an unknown and non-linear reward function. This complexity makes direct extensions of existing quantum algorithms infeasible. 
Therefore, novel methodologies are required to effectively bridge quantum computing with the Lipschitz bandit framework.

\begin{table*}[t]
\caption{Comparison of regret bounds on Lipschitz bandits.}\label{table:regret}
\centering
\small{
\begin{sc}
\begin{tabular}{lcccc}
\toprule
Algorithm & Reference & Setting & Noise & Regret Bound \\
\midrule
\multicolumn{2}{l}{\cellcolor{gray!30}}&  Classical& sub-Gaussian & $ O \left( T^{\frac{d_z+1}{d_z+2}}  (\log T)^{ \frac{1}{d_z+2}} \right)$\\
\midrule
\multirow{2}{*}{\centering Q-LAE}  & Theorem~\ref{thm:blin} & Quantum & bounded value & $ O \left( T^{\frac{d_z}{d_z+1}}  (\log T)^{ \frac{2}{d_z+1}} \right)$  \\
    & Theorem~\ref{thm:blin_bounded_variance} &  Quantum & bounded variance & $ O \left( T^{\frac{d_z}{d_z+1}}  (\log T)^{\frac72 \frac{1}{d_z+1}} (\log\log T)^{\frac{1}{d_z+1}} \right)$\\
\midrule
\multirow{2}{*}{\centering Q-Zooming} & Theorem~\ref{thm:q-zooming} & Quantum &  bounded value & $ O \left( T^{\frac{d_z}{d_z+1}}  (\log T)^{ \frac{1}{d_z+1}} \right)$ \\
   & Theorem~\ref{thm:zooming_bounded_variance} & Quantum & bounded variance & $ O \left( T^{\frac{d_z}{d_z+1}}  (\log T)^{\frac52 \frac{1}{d_z+1}} (\log\log T)^{\frac{1}{d_z+1}} \right)$ \\
\bottomrule
\end{tabular}
\end{sc}
}
\end{table*}

In this work, we propose the first two effective quantum Lipschitz bandit algorithms that achieve improved regret bounds by employing quantum computing techniques. Our main contributions are summarized as follows:

\begin{itemize}
    \item We propose an elimination-based algorithm named \textit{Quantum Lipschitz Adaptive Elimination} (Q-LAE), which achieves an improved regret bound of order $\tilde O\left(T^{{d_z}/{(d_z+1)}}\right)$ under two standard noise assumptions. The algorithm leverages the concepts of covering and maximal packing, making it applicable to general metric spaces. Compared to the existing elimination-based method, our algorithm adopts a more rigorous definition of the zooming dimension $d_z$, detailed in Section~\ref{sec:Q-LAE}.
    \item We introduce the \textit{Quantum Zooming} (Q-Zooming) Algorithm, which extends the classical zooming algorithm to the quantum Lipschitz bandit setting, also achieving the regret bound of order $\tilde O\left(T^{{d_z}/{(d_z+1)}}\right)$. This extension is non-trivial, requiring substantial modifications such as a stage-based design that efficiently leverages quantum oracles and integrates the QMC method.
    \item We evaluate the performance of our proposed algorithms, Q-LAE and Q-Zooming, through numerical experiments. The results demonstrate that both algorithms consistently outperform the classical Zooming algorithm across a variety of  Lipschitz bandit scenarios.
\end{itemize}
Table~\ref{table:regret} outlines the regret bounds achieved by our methods along with state-of-the-art classical Lipschitz bandit algorithms under different noise assumptions.

\section{Related Work}

This section reviews prior work on quantum online learning, including quantum bandits and quantum reinforcement learning. Due to space constraints, related work on Lipschitz bandits is deferred to Appendix C.

Prior work on bandit algorithms has made significant strides in integrating quantum computing into various bandit problems. \citet{wan2023quantum} first incorporated quantum computing into the multi-armed and stochastic linear bandits with the linear reward model and showed that an improved regret bound over the classical lower bound could be attained. \citet{wu2023quantum} extended this line of research to quantum bandits with heavy-tailed rewards, while \citet{li2022quantum} explored the quantum stochastic convex optimization problem. The best-arm identification problem in quantum multi-armed bandits has also been studied in \cite{wang2021quantum,casale2020quantum}. Parallel to quantum bandits, significant progress has been made in quantum reinforcement learning, particularly under linear reward models, as summarized in the comprehensive survey by \citet{meyer2022survey}. For instance, \citet{wang2021quantumrl} showed that their quantum reinforcement learning algorithm could achieve better sample complexity than classical methods under a generative model of the environment. Furthermore, \citet{ganguly2023quantum} achieved improved regret bounds in the quantum setting without relying on a generative model. To deal with non-linear reward models, \citet{dai2024quantum} proposed a simple quantum algorithm for kernelized bandits by improving the confidence ellipsoid, and then an improved quantum kernelized bandit method was developed in~\cite{hikima2024quantum}. However, unlike kernelized bandits, which leverage well-defined kernel functions to structure the action space, Lipschitz bandits face additional challenges due to the inherent complexity of non-linear reward functions and the unstructured, diverse nature of arm metric spaces. As a result, the application of quantum computing to Lipschitz bandits remains an open and significantly more difficult problem, one that we aim to address in this work.

\section{Problem Setting and Preliminaries}\label{sec:prelim}

In this section, we present the problem setting of the Lipschitz bandit problem and provide essential background on quantum computation, forming the foundation for the quantum bandit framework and our methods.

\subsection{Lipschitz Bandits} 
The Lipschitz bandit problem can be characterized by a triplet \( (X, \mathcal{D}, \mu) \), where \( X \) denotes the action space of arms, \( \mathcal{D} \) is a metric on \( X \), and \( \mu : X \to \mathbb{R} \) represents the unknown expected reward function. The reward function \( \mu \) is assumed to be \( 1 \)-Lipschitz with respect to the metric space $(X,\mathcal{D})$, which implies:
\begin{equation*}
    |\mu(x_1) - \mu(x_2)| \leq \mathcal{D}(x_1, x_2), \quad \forall x_1, x_2 \in X.
\end{equation*}
Without loss of generality, we assume $(X,\mathcal{D})$ is a compact doubling metric space with its diameter no more than $1$.

At each time step \( t \leq T \), the agent selects an arm \( x_t \in X \). The stochastic reward associated with the chosen arm is drawn from an unknown distribution \( P_x \) and is observed as $y_t = \mu(x_t) + \eta_t$,
where \( \eta_t \) is independent noise with zero mean and finite variance.
As in standard bandit settings, the agent’s objective is to minimize the cumulative regret over \( T \) rounds, defined as $R(T) = \sum_{t=1}^T (\mu^* - \mu(x_t))$,
where \( \mu^* = \max_{x \in X} \mu(x) \) denotes the maximum expected reward. Additionally, the \textit{optimality gap} for an arm \( x \in X \) is given by \( \Delta_x = \mu^* - \mu(x) \).

\paragraph{Zooming Dimension.}

The zooming dimension~\cite{kleinberg2019bandits,  bubeck2008online} is a fundamental concept in Lipschitz bandit problems over metric spaces. It characterizes the complexity of the problem by accounting for both the geometric structure of the arm set and the behavior of the reward function \( \mu(\cdot) \). A key component of this concept is the \textit{zooming number} \( N_z(r) \), which represents the minimal number of radius-\( r/3 \) balls required to cover the set of near-optimal arms: 
\begin{equation}
    X_r = \{x \in X : r \leq \Delta_x < 2r\},
\end{equation}
where each arm has an optimality gap between \( r \) and \( 2r \).

Building on this, the \textit{zooming dimension} \( d_z \) is introduced to quantify the growth rate of the zooming number \( N_z(r) \). Formally, it is defined as:
\[
d_z := \min \{ d \geq 0 : \exists \alpha > 0, \, N_z(r) \leq \alpha r^{-d}, \, \forall r \in (0, 1] \}.
\]
Unlike the covering dimension $d_c$~\cite{kleinberg2008multi}, which depends solely on the metric structure of the arm set \( X \), the zooming dimension \( d_z \) captures the interaction between the metric and the reward function \( \mu(\cdot) \). While the covering dimension characterizes the entire metric space, the zooming dimension focuses specifically on the near-optimal subset of \( X \), which often results in \( d_z \) being substantially smaller than \( d_c \). Notably, the zooming dimension \( d_z \) is not directly available to the agent, as it depends on the unknown reward function \( \mu(\cdot) \). This dependence makes designing algorithms based on \( d_z \) particularly challenging.

\paragraph{Covering and Packing.} 
We discuss the concepts of covering, packing, and maximal packing, which are fundamental to the algorithm proposed in Section~\ref{sec:Q-LAE}.
Let \((X, \mathcal{D})\) be a metric space with a subset \(S \subseteq X\).

\begin{definition}
 A set of points \(\{x_1, ..., x_n\} \subseteq S\) is an \textit{\(\epsilon\)-covering} of \(S\) if $ S \subseteq \bigcup_{i=1}^{n} \mathcal{B}(x_i, \epsilon)$, where \(\mathcal{B}(x, \epsilon) = \{y \in X \, | \, \mathcal{D}(x, y) \leq \epsilon\}\) denotes the closed ball of radius \(\epsilon\) centered at \(x\).    
\end{definition}

\begin{definition}
A set of points \(\{x_1, ..., x_n\} \subseteq S\) is an \textit
{\(\epsilon\)-packing} of \(S\) if $\mathcal{D}(x_i, x_j) \geq \epsilon$ for all $ i \neq j$.
\end{definition}

\begin{definition}
An \(\epsilon\)-packing \(\{x_1, ..., x_n\}\) is called a \textit{maximal \(\epsilon\)-packing} if no additional point \(x \in S\) can be added without violating the packing condition.
\end{definition}

\subsection{Quantum Computation}\label{subsec:quantum_computing}

\paragraph{Basics of Quantum States and Measurements.}

In quantum mechanics, superposition is a fundamental concept that describes how a quantum system can simultaneously exist in multiple states until it is measured or observed. This contrasts with classical systems, where a system can only exist in a single state at a time.
A \textit{quantum state} in a Hilbert space \( \mathbb{C}^n \) is represented as an \( L^2 \)-normalized column vector \( \mathbf{x} = [x_1, x_2, \ldots, x_n]^\top \), where \( |x_i|^2 \) denotes the probability of being in the \( i \)-th basis state. This superposition of states is expressed using Dirac notation as \( |\mathbf{x}\rangle \). The conjugate transpose \( \mathbf{x}^\dagger \) is written as \( \langle \mathbf{x} | \), forming the bra-ket notation.

Given two quantum states \( |\mathbf{x}\rangle \in \mathbb{C}^n \) and \( |\mathbf{y}\rangle  \in \mathbb{C}^m \), their joint quantum state is expressed by the tensor product $|\mathbf{x}\rangle |\mathbf{y}\rangle$, which can be written explicitly as \(  [x_1 y_1, \ldots, x_n y_m]^\top \in \mathbb{C}^{nm} \). Quantum mechanics does not offer full knowledge of a state vector, but only allows partial information to be accessed through \textit{quantum measurements}. These measurements are typically represented by a set of positive semi-definite Hermitian matrices \( \{E_i\}_{i \in \Lambda} \), where $\Lambda$ is the set of possible outcomes, and \( \sum_{i \in \Lambda} E_i = I \). The probability of observing the outcome \( i \) is given by \( \langle \mathbf{x} | E_i | \mathbf{x} \rangle \), which represents the inner product between \( \langle \mathbf{x} | \) and \( E_i | \mathbf{x} \rangle \). The condition \( \sum_{i \in \Lambda} E_i = I \) ensures that the total probability of all possible outcomes equals 1. Once a measurement is performed, the original quantum state collapses into the specific state associated with the observed outcome.

\paragraph{Quantum Reward Oracle and Bandit Settings.}

Quantum algorithms process quantum states using \textit{unitary operators}, which are referred to as \textit{quantum reward oracles}. These oracles encode information about the reward distribution for a given action, replacing the immediate sample reward used in classical bandit settings.

In the quantum bandit framework, at each round, the learner selects an action \( x \) and gains access to the quantum unitary oracle \( \mathcal{O}_{x} \). This oracle encodes the reward distribution \( P_x \) associated with the chosen action \( x \) and is formally defined as:
$
\mathcal{O}_x : |0\rangle \to \sum_{\omega \in \Omega_x} \sqrt{P_x(\omega)} |\omega\rangle |y^x(\omega)\rangle
$,
where $\Omega_x$ is the sample space of $P_x$ and \( y^x : \Omega_x \to \mathbb{R} \) represents the random reward corresponding to arm \( x \). This unitary operator enables quantum algorithms to interact with reward distributions in a fundamentally different way compared to classical approaches.

\paragraph{Quantum Monte Carlo (QMC) Method.}

Estimating the mean of an unknown reward distribution is a crucial task in bandit problems. In pursuit of quantum speedup, we employ the Quantum Monte Carlo method~\cite{montanaro2015quantum}, which is designed for efficient mean estimation of unknown distributions. This approach demonstrates enhanced sample efficiency compared to classical methods.

\begin{lemma}[Quantum Monte Carlo method~\cite{montanaro2015quantum}]\label{lem:qmc}
Let \( y : \Omega \to \mathbb{R} \) be a random variable with bounded variance, where \( \Omega \) is equipped with a probability measure \( P \), and the quantum unitary oracle \( \mathcal{O} \) encodes \( P \) and \( y \). With certain assumption about the noise, the QMC method offers the following guarantees:
\begin{itemize}
    \item Bounded Noise: If \( y \in [0, 1] \), there exists a constant \( C_1 > 1 \) and a quantum algorithm \( \text{QMC}_1(\mathcal{O}, \epsilon, \delta) \) that outputs an estimate \( \hat{y} \) of \( \mathbb{E}[y] \), such that 
    \(
    \Pr(|\hat{y} - \mathbb{E}[y]| \geq \epsilon) \leq \delta,
    \)
    while requiring at most \( \frac{C_1}{\epsilon} \log \frac{1}{\delta} \) queries to \( \mathcal{O} \) and \( \mathcal{O}^\dagger \).
    \item Noise with Bounded Variance: If \( \text{Var}(y) \leq \sigma^2 \), then for \( \epsilon < 4\sigma \), there exists a constant \( C_2 > 1 \) and a quantum algorithm \( \text{QMC}_2(\mathcal{O}, \epsilon, \delta) \) that outputs an estimate \( \hat{y} \) of \( \mathbb{E}[y] \), such that 
    $
    \Pr(|\hat{y} - \mathbb{E}[y]| \geq \epsilon) \leq \delta,
    $
    while requiring at most 
    \(
    \frac{C_2 \sigma}{\epsilon} \log^{3/2}_2\left(\frac{8\sigma}{\epsilon}\right) \log_2\left(\log_2\frac{8\sigma}{\epsilon}\right) \log \frac{1}{\delta}
    \)
    queries to \( \mathcal{O} \) and \( \mathcal{O}^\dagger \).
\end{itemize}
\end{lemma}
From Lemma~\ref{lem:qmc}, we observe a notable quadratic improvement in sample complexity with respect to $\epsilon$ when estimating \( \mathbb{E}[y] \). While classical Monte Carlo methods require \( \widetilde{O}\left(1/\epsilon^2\right) \) samples, QMC achieves the same range of confidence interval with only \( \widetilde{O}\left(1/\epsilon \right) \) samples. 
This substantial reduction in sample complexity is a fundamental advantage of quantum computing, which we would utilize in our following proposed methods.

\section{Quantum Lipschitz Adaptive Elimination Algorithm}\label{sec:Q-LAE}

The structure of the QMC algorithm requires repeatedly playing a specific arm multiple times to obtain an updated reward estimate. The algorithm continues to play the same arm until a sufficient number of samples is collected, without updating its estimate. Inspired by the nature of the QMC algorithm and recent advancements in elimination-based algorithms, we propose a novel method called Quantum Lipschitz Adaptive Elimination (Q-LAE), which adapts the batched elimination framework and leverages the power of quantum machine learning.

Unlike existing elimination-based Lipschitz bandit algorithms, our approach adopts the consistent definition of the zooming dimension introduced by \cite{kleinberg2019bandits, slivkins2019introduction}, as presented in Section~\ref{sec:prelim}. By doing so, it achieves improved regret bounds through the integration of quantum machine learning. In contrast, existing algorithms~\cite{feng2022lipschitz, kang2023robust} rely on an alternative definition of the zooming dimension, which may be larger than the classical one and can lead to looser bounds in some cases. A detailed discussion on this issue is provided in the following Remark~\ref{remark:zooming_dimension}. \textbf{To the best of our knowledge, this is the first elimination-based Lipschitz bandit algorithm that adheres to the consistent definition of the zooming dimension.}

\begin{remark}\label{remark:zooming_dimension}
In the Lipschitz bandit literature, two distinct definitions of the zooming dimension are commonly used. Specifically, \citet{kleinberg2019bandits} and \citet{slivkins2019introduction} define the zooming number as the minimal number of balls required to cover the set of near-optimal arms \( X_r = \{x \in X : r \leq \Delta_x < 2r\} \), whereas \citet{feng2022lipschitz} defines it based on a broader set \( Y_r = \{x \in X : \Delta_x \leq 2r\} \). These differing definitions can yield substantially different values.

For example, if the expected reward is constant across the entire arm space, then under our classical definition, \( X_r \) is empty for all \( r \), resulting in a zooming dimension of 0. In contrast, \( Y_r \) equals the entire space \( X \) for all \( r \), implying that the zooming dimension is equivalent to the covering dimension of the space. In this work, we propose the first elimination-based Lipschitz bandit algorithm that adopts the former (classical) definition of the zooming dimension.
\end{remark}

To achieve accurate reward estimation for each ball, our algorithm selects a single representative arm, which is the center arm of each cube, and plays it multiple times. By utilizing mean reward estimates from QMC, Q-LAE adaptively concentrates on high-reward regions in continuous spaces through a combination of \textit{selective elimination} and \textit{progressive refinement}. This approach significantly reduces exploration in less promising areas, enhancing both efficiency and overall performance.

To explain the algorithm in detail, consider the beginning of each stage \( m \), where the algorithm receives \(\mathcal{A}_m\), a maximal \(\epsilon_m\)-packing of the current active arm region \(\mathcal{C}_m\). For every point in \(\mathcal{A}_m\), the algorithm performs \( n_m  \) plays and uses the QMC algorithm to estimate the mean rewards. Based on these estimates, the algorithm applies \textit{selective elimination} to discard low-performing regions. Specifically, any point \( x \) at stage $m$ is eliminated if its estimated mean reward \(\hat{\mu}_m(x)\) is more than \( 3\epsilon_m \) lower than the best estimated reward among all points in \(\mathcal{A}_m\). The set of remaining points after this elimination step is denoted by \(\mathcal{A}_m^+\).

Next, the algorithm performs \textit{progressive refinement}, which incrementally narrows the search space. The remaining active region is refined by further subdividing areas with higher potential rewards. More precisely, the active region is updated as
\begin{equation}
    \mathcal{C}_{m+1} \leftarrow \bigcup_{x \in \mathcal{A}_m^+} \mathcal{B}(x, \epsilon_m).
\end{equation}
The updated region \(\mathcal{C}_{m+1}\) is then discretized through its maximal \(\epsilon_{m+1}\)-packing, forming the next set of active points \(\mathcal{A}_{m+1}\). This iterative process allows the algorithm to progressively concentrate exploration on more promising regions of the arm space. Notably, since a maximal \(\epsilon\)-packing is an \(\epsilon\)-covering, we can utilize maximal packing as representative points for the active region. This mathematical fact is detailed in Appendix A. The complete learning procedure is outlined in Algorithm~\ref{alg:Q-LAE}.

\begin{algorithm}[t]
\caption{Q-LAE Algorithm}\label{alg:Q-LAE}
\textbf{Input:} time horizon $T$, fail probability $\delta$ \\
\textbf{Initialization:} $\mathcal{A}_1 \leftarrow \text{maximal-} \frac12 \text{ packing of } X$, \\
 $\mathcal{C}_1 \leftarrow X$, $\epsilon_m = 2^{-m}$ for all $m$
\begin{algorithmic}[1]
\FOR{stage $m = 1,2,...$}
\STATE $n_m \leftarrow \frac{C_1}{\epsilon_m} \log \left( \frac{T}{\delta} \right) $
\FOR{each $x \in \mathcal{A}_m$}
\STATE Play $x$ for the next $n_m$ rounds and obtain $\hat{\mu}_m(x)$ by running the $\text{QMC}_1(\mathcal{O}_{x},\epsilon_m,\delta/T)$ algorithm. Terminate the algorithm if the number of rounds played exceeds $T$.
\ENDFOR
\LineComment{selective elimination}
\STATE $\hat{\mu}_{max} \leftarrow \max_{x \in \mathcal{A}_m}\hat{\mu}_m(x)$
\STATE For each $x \in \mathcal{A}_m$, eliminate $x$ if $\hat{\mu}_m(x) < \hat{\mu}_{max} - 3\epsilon_m$. Let $\mathcal{A}_m^+$ denote the set of points not eliminated.
\LineComment{progressive refinement}
\STATE $\mathcal{C}_{m+1} \leftarrow \bigcup_{x \in \mathcal{A}_m^+} \mathcal{B}(x, \epsilon_m)$
\STATE Find a maximal $\epsilon_{m+1}$-packing of $\mathcal{C}_{m+1}$ and define it as $\mathcal{A}_{m+1}$.
\ENDFOR
\end{algorithmic}
\end{algorithm}

\subsection{Regret Analysis}\label{subsec:Q-LAE_regret}

Here, we analyze the regret of the Q-LAE algorithm. In the bandit literature, a \textit{clean event} refers to the scenario where the reward estimates for all active arms and stages are bounded within their confidence intervals with high probability. While the classical (non-quantum) setting typically uses the Chernoff bound to establish deviation inequality between the true reward and its estimate, our quantum setting relies on the QMC algorithm.

At the end of line 4 in Algorithm~\ref{alg:Q-LAE}, Lemma~\ref{lem:qmc} ensures that the QMC algorithm provides an estimate \(\hat{\mu}_m(x)\) that satisfies $| \hat{\mu}_m(x) - \mu(x) | \leq \epsilon_m$ with probability at least \(1 - {\delta}/{T}\).
Since every arm \(x \in \mathcal{A}_m\) is played at least once in each stage \(m\), applying the union bound over all stages and arms guarantees that
$
    |\hat{\mu}_m(x) - \mu(x) | \leq \epsilon_m
$
holds for all stages $m$ and \(x \in \mathcal{A}_m\) with probability at least \(1 - \delta\). We define this as the clean event for this section and assume it holds throughout the subsequent regret analysis.

We now present a theorem that provides an upper bound on the regret of the Q-LAE algorithm, with the detailed proof provided in Appendix A.1.

\begin{restatable}{theorem}{thmblin}
\label{thm:blin}
Under the bounded noise assumption, the cumulative regret \( R(T) \) of the Q-LAE Algorithm is bounded with high probability, at least \( 1 - \delta \), as follows:
\[
R(T) = O \left( T^{\frac{d_z}{d_z+1}}  \,(\log T)^{\frac{2}{d_z+1}} \right),
\]
where \( d_z \) represents the zooming dimension of the problem instance.
\end{restatable}

Theorem~\ref{thm:blin} establishes that our Q-LAE algorithm achieves a regret bound of $\tilde O\left(T^{d_z/(d_z+1)}\right)$, significantly improving upon the optimal regret bound of $\tilde O\left(T^{(d_z+1)/(d_z+2)}\right)$ attained by classical algorithms. Since the zooming dimension \(d_z\) is a small nonnegative number, often substantially smaller than the covering dimension \(d_c\), this represents a major improvement. A more detailed discussion of the zooming dimension can be found in Appendix D.

Additionally, we provide a theorem that analyzes the regret of the Q-LAE algorithm under the assumption of bounded noise variance. In this setting, with a different choice of $n_m$, the algorithm's regret bound of \(\tilde{O}\left(T^{{d_z}/{(d_z + 1)}}\right)\) matches the result for the bounded noise setting in Theorem~\ref{thm:blin} up to logarithmic factors. More details and the analysis of Theorem~\ref{thm:blin_bounded_variance} are deferred to Appendix A.2.

\begin{restatable}{theorem}{thmblinvariance}
\label{thm:blin_bounded_variance}
With the choice of $n_m = \frac{C_2 \sigma}{\epsilon_m} \log^{3/2}_2\left(\frac{8\sigma}{\epsilon_m}\right) \log_2\left(\log_2\frac{8\sigma}{\epsilon_m}\right) \log\left( \frac{T}{\delta}\right)$ in line 2, under the bounded variance noise assumption, the cumulative regret $R(T)$ of the Q-LAE Algorithm is bounded with high probability, at least \( 1 - \delta \), as follows:
\[
R(T) = O \left( T^{\frac{d_z}{d_z+1}}  (\log T)^{\frac72 \frac{1}{d_z+1}} (\log\log T)^{\frac{1}{d_z+1}} \right),
\]
where \( d_z \) is the zooming dimension of the problem instance.
\end{restatable}

\section{Quantum Zooming Algorithm}\label{sec:q-zooming}

Next, we introduce our second methodology, the Quantum Zooming (Q-Zooming) algorithm. This algorithm is a novel adaptation of the classical Zooming algorithm~\cite{kleinberg2019bandits}, which efficiently focuses on promising regions of the arm space by adaptively discretization. Our Q-Zooming algorithm achieves an improved regret bound compared to our first algorithm, Q-LAE. While inspired by its classical counterpart, our algorithm is not a straightforward extension. Unlike the classical version, the Q-Zooming algorithm operates in stages, and follows two primary rules: the \textit{activation rule} and the \textit{selection rule}. In this section, we describe these two rules and highlight the novel modifications introduced in our Q-Zooming algorithm.

\begin{algorithm}[t]
\caption{Q-Zooming Algorithm}\label{alg:q-zooming}
\textbf{Input:} time horizon $T$, fail probability $\delta$ \\
\textbf{Initialization:} active arms set $S \leftarrow \emptyset$, confidence radius $\epsilon_0(\cdot)=1$
\begin{algorithmic}[1]
\FOR{stage $s = 1,2,...,m$}
\LineComment{activation rule}
\IF{there exists an arm $y$ that is not covered by the confidence balls of active arms}
\STATE add $y$ to the active set: $S \leftarrow S \cup \{ y \}$
\ENDIF
\LineComment{selection rule}
\STATE $x_s \leftarrow \argmax_{x \in S} \hat{\mu}_{s-1}(x)+ 2 \epsilon_{s-1}(x) $
\STATE $\epsilon_s(x) \leftarrow \begin{cases}
    \epsilon_{s-1}(x) /2  \quad \text{if } x = x_s, \\
    \epsilon_{s-1}(x)  \quad \text{if } x \neq x_s.
\end{cases}$
\STATE $N_{s} \leftarrow \frac{C_1}{\epsilon_s(x_s)} \log \left( \frac{m}{\delta} \right) $
\STATE If $\sum_{k=1}^{s} N_s > T$, terminate the algorithm.
\STATE Play $x_s$ for the next $N_s$ rounds and obtain $\hat{\mu}_s(x_s)$ by running the $\text{QMC}_1(\mathcal{O}_{x_s},\epsilon_s(x_s),\delta/m)$ algorithm.\\
\quad \textit{Note: For all other arms \( x \neq x_s \), retain the reward estimates \( \hat{\mu}_{s-1}(x) \) from stage \( s-1 \).}
\ENDFOR
\end{algorithmic}
\end{algorithm}

The \textit{activation rule} ensures that any arm not covered by the confidence balls of the active arms is added to the active set. Formally, an arm \( y \) is considered covered at stage \( s \) if there exists $x \in S$ such that $\mathcal{D}(x, y) \leq \epsilon_{s-1}(x)$, where \( S \) is the current set of active arms, \( \mathcal{D}(x, y) \) is the distance between arms \( x \) and \( y \), and \( \epsilon_{s-1}(x) \) is the confidence radius for arm \( x \) updated at stage \( s-1 \). If such an uncovered arm \( y \) exists, it is added to the active set $S$.

The \textit{selection rule} determines which arm to sample next from the active set. Specifically, it selects the arm \( x \in S \) with the highest index, defined as $\hat{\mu}_{s-1}(x)+ 2 \epsilon_{s-1}(x)$, where \( \hat{\mu}_{s-1}(x) \) denotes the estimated mean reward of arm \( x \) and \( \epsilon_{s-1}(x) \) is its corresponding confidence radius.

Our Q-Zooming algorithm combines these activation and selection rules with several novel techniques to adapt to the quantum computing setting. Unlike the classical approach, where a single sample is played once per round, our algorithm leverages QMC for reward estimation. This fundamental difference requires multiple queries to the quantum oracle of a selected arm to accurately estimate its mean reward. To handle this, we partition the time horizon \(T\) into multiple stages. At each stage, the algorithm selects the arm with the highest index and allocates multiple queries to estimate its reward with high precision. Additionally, as described in line 8 of Algorithm~\ref{alg:q-zooming}, the confidence radius shrinks by half at each stage only when the corresponding arm is selected. This guarantees that the total number of plays for any given arm remains bounded, a property that is crucial for the analysis of our algorithm. We also show that the optimality gap of any active arm is bounded by its confidence radius. Since this confidence radius shrinks only when the arm is selected, it implies that, in later stages, only arms with rewards close to that of the optimal arm are repeatedly selected. This property is fundamental to the algorithm's effectiveness.

The procedure of the Q-Zooming algorithm is detailed in Algorithm~\ref{alg:q-zooming}. It begins with an empty set of active arms \(S\) and an initial confidence radius of 1 for all arms. At each stage \(s\), the algorithm activates an arm that is not covered by the confidence balls and selects the best arm \(x_s\) based on the activation and selection rules described earlier. The confidence radius \(\epsilon_s(x_s)\) for the selected arm \(x_s\) is then updated by halving its value. Using this updated \(\epsilon_s(x_s)\), the algorithm determines the number of queries \(N_s\) required for the quantum oracle. The algorithm then plays the selected arm \(x_s\) for \(N_s\) rounds and updates the reward estimate \(\hat{\mu}_s(x_s)\) using the QMC algorithm. This step provides a more accurate reward estimate, bounding the estimation error such that \(|\hat{\mu}_s(x_s) - \mu_s(x_s)| \leq \epsilon_s(x_s)\) with probability at least \(1 - \delta/m\), as established in Lemma~\ref{lem:qmc}. Once these steps are completed, the algorithm proceeds to the next stage unless the total number of rounds exceeds the time horizon \(T\), in which case the algorithm terminates.

\subsection{Regret Analysis}\label{subsec:Q-zooming_regret}

We provide a regret analysis for the proposed Q-Zooming algorithm. Similar to the approach described in Section~\ref{subsec:Q-LAE_regret}, we utilize the clean event approach to structure the analysis.

Following each stage, Lemma~\ref{lem:qmc} guarantees that QMC, with a sufficient number of queries \( N_s \), provides an estimate \(\hat{\mu}_s(x_s)\) satisfying \(|\hat{\mu}_s(x_s) - \mu(x_s)| \leq \epsilon_s(x_s)\) with probability at least \(1 - \frac{\delta}{m}\), where \( m \) is the total number of stages. By applying the union bound over all stages $1 \leq s \leq m$, it follows that \(|\hat{\mu}_s(x_s) - \mu(x_s)| \leq \epsilon_s(x_s)\) holds for all \( 1 \leq s \leq m \) with probability at least \(1 - \delta\).

Additionally, this result extends to all arms $x$. Specifically, $|\hat{\mu}_s(x) - \mu(x)| \leq \epsilon_s(x)$ holds for all \( 1 \leq s \leq m \) and every arm \( x \), with probability at least \(1 - \delta\). This is because $\epsilon_s(x)$ is initially set to 1 for arms that have never been played, and both \(\epsilon_s(x)\) and \(\hat{\mu}_s(x)\) remain unchanged for arms not selected during stage s. We define this condition as a clean event and assume it holds throughout the subsequent analysis. We now establish the following high-probability upper bound on the regret for the Q-Zooming algorithm. The proof is deferred to Appendix B.1.

\begin{restatable}{theorem}{thmzooming}
\label{thm:q-zooming}
Under the bounded noise assumption, the cumulative regret \( R(T) \) of the Q-Zooming Algorithm is bounded with high probability, at least \( 1 - \delta \), as follows:
\[
R(T) = O \left( T^{\frac{d_z}{d_z+1}} \, ( \log T)^{\frac{1}{d_z+1}} \right),
\]
where \( d_z \) is the zooming dimension of the problem instance.
\end{restatable}

Theorem~\ref{thm:q-zooming} shows that the Q-Zooming algorithm attains a regret bound of \(\tilde{O}\left(T^{d_z/(d_z+1)}\right)\), surpassing the optimal regret bound of classical Lipschitz bandit algorithms, \(\tilde{O}\left(T^{(d_z+1)/(d_z+2)}\right)\). This enhancement is similar to that achieved by the Q-LAE algorithm in Section~\ref{sec:Q-LAE}. As noted in Section~\ref{sec:Q-LAE} and Appendix D, since \(d_z\) is often much smaller than \(d_c\), this improvement is particularly significant.

\begin{remark}\textit{
The regret bounds for the Q-Zooming and Q-LAE algorithms are identical, ignoring polylogarithmic factors. However, when considering the \(\log T\) term, Q-Zooming exhibits a tighter bound. The key difference lies in their strategies: Q-LAE focuses on completely eliminating low-reward regions, while Q-Zooming, without eliminating arms, adopts a strategy of further discretizing high-reward regions to explore them more effectively. Despite their theoretical regret bounds, the practical performance of these two algorithms can vary depending on how quickly Q-LAE eliminates low-reward regions, as we illustrate in Section~\ref{sec:experiments}.
}
\end{remark}

Under the bounded variance assumption, using a different choice of $N_s$, the regret bound remains consistent with Theorem~\ref{thm:q-zooming}, differing only by logarithmic factors. A detailed analysis for the bounded variance case, including the proof of Theorem~\ref{thm:zooming_bounded_variance}, is provided in Appendix B.2.

\begin{restatable}{theorem}{thmzoomingvariance}
\label{thm:zooming_bounded_variance}
With the choice of 
$N_{s} = \frac{C_2 \sigma}{\epsilon_s(x_s)} \log^{3/2}_2\left(\frac{8\sigma}{\epsilon_s(x_s)}\right) \log_2\left(\log_2\frac{8\sigma}{\epsilon_s(x_s)}\right) \log\left( \frac{m}{\delta}\right)$ in line 9,
the cumulative regret \( R(T) \) of the Q-Zooming Algorithm under the bounded variance noise assumption is bounded with high probability, at least \( 1 - \delta \), as follows:
\[
R(T) = O \left( T^{\frac{d_z}{d_z+1}}  (\log T)^{\frac52 \frac{1}{d_z+1}} (\log\log T)^{\frac{1}{d_z+1}} \right),
\]
where \( d_z \) is the zooming dimension of the problem instance.
\end{restatable}

\begin{figure*}[t]
\centering
\begin{subfigure}{0.28\textwidth}
\includegraphics[trim={0 3mm 1.4cm 1.3cm},clip,width=\textwidth]{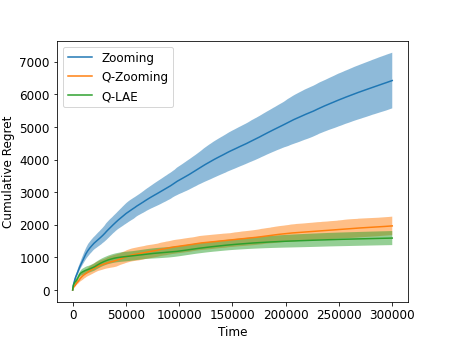}
\caption{Triangle (Bernoulli)}
\end{subfigure}
\begin{subfigure}{0.28\textwidth}
\includegraphics[trim={0.7cm 3mm 1.3cm 1.3cm},clip,width=0.96\textwidth]{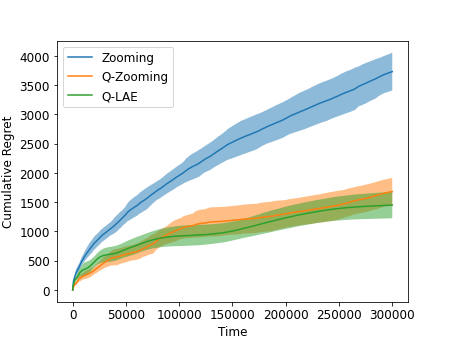}
\caption{Sine (Bernoulli)}
\end{subfigure}
\begin{subfigure}{0.28\textwidth}
\includegraphics[trim={0.45cm 3mm 1.4cm 1.3cm},clip, width=0.97\textwidth]{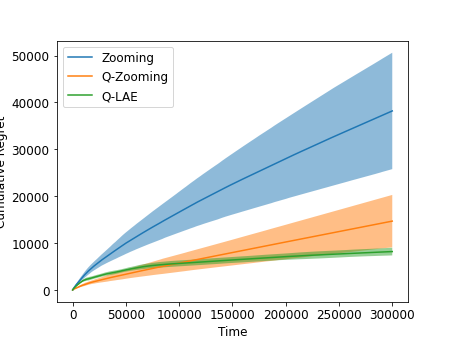}
\caption{Two-Dimensional (Bernoulli)}
\end{subfigure}
\begin{subfigure}{0.28\textwidth}
\includegraphics[trim={0 3mm 1.4cm 1.3cm},clip,width=\textwidth]{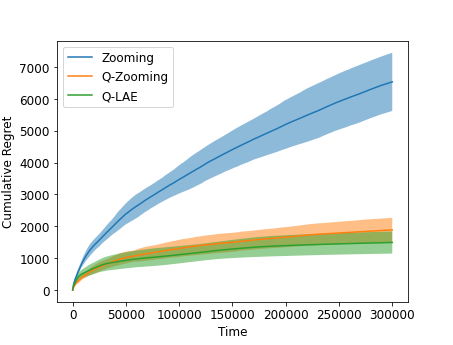}
\caption{Triangle (Gaussian)}
\end{subfigure}
\begin{subfigure}{0.28\textwidth}
\includegraphics[trim={0.7cm 3mm 1.3cm 1.3cm},clip,width=0.96\textwidth]{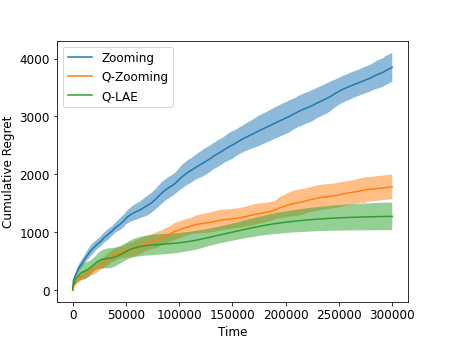}
\caption{Sine (Gaussian)}
\end{subfigure}
\begin{subfigure}{0.28\textwidth}
\includegraphics[trim={0.45cm 3mm 1.4cm 1.3cm},clip, width=0.97\textwidth]{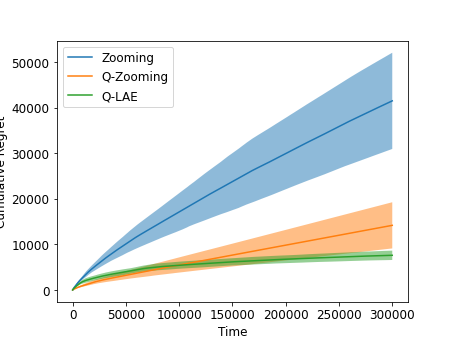}
\caption{Two-Dimensional (Gaussian)}
\end{subfigure}
\caption{
Average Cumulative Regret. The regret performance for each function and noise model is evaluated as the average cumulative regret over 30 independent runs, with the corresponding standard deviations also reported.
}
\label{fig:experiment}
\end{figure*}

\section{Experiments}\label{sec:experiments}

We evaluate the regret performance of our proposed quantum algorithms, Q-LAE and Q-Zooming, against the classical Zooming algorithm~\cite{kleinberg2019bandits} on Lipschitz bandit problems. Specifically, we conduct experiments on three Lipschitz functions, following the common settings in studies on Lipschitz bandit problems~\cite{kang2023robust, feng2022lipschitz}: (1) \(\mu(x) = 0.9 - 0.95|x - 1/3|\) with \((X, \mathcal{D}) = ([0, 1], |\cdot|)\) (triangle), (2)  \(\mu(x) = 0.35 \sin\left(3\pi x / 2\right)\) with \((X, \mathcal{D}) = ([0, 1], |\cdot|)\) (sine), and (3) \(\mu(x) = 1.2 - 0.95\|x - (0.8, 0.7)\|_2 - 0.3\|x - (0, 1)\|_2\) with \((X, \mathcal{D}) = ([0, 1]^2, \|\cdot\|_\infty)\) (two-dimensional). In all three cases, the reward function \(\mu(x)\) is bounded within the interval \([0, 1]\). We also consider two types of noise, as described in Lemma~\ref{lem:qmc} and aligned with our theoretical analysis: (a) bounded noise, where the output \(y\) is modeled as a Bernoulli random variable with \(\mu(x)\) as the probability of success (\(y = 1\)), and (b) noise with bounded variance, where zero-mean Gaussian noise with variance \(\sigma^2 \) is added directly to \(\mu(x)\) as the observed reward.

The QMC algorithm and our proposed quantum algorithms, Q-LAE and Q-Zooming, are implemented using the Python package Qiskit~\cite{qiskit2024}. In our experiments, we set the time horizon to \(T = 300,000\) and the failure probability to \(\delta = 0.05\). Gaussian noise is sampled from a normal distribution \(N(0, \sigma^2 = 0.1)\). We evaluate performance by averaging the cumulative regret over 30 independent trials. Both the mean and standard deviation of the cumulative regret are reported in Figure~\ref{fig:experiment}.

Figure~\ref{fig:experiment} shows that both Q-LAE and Q-Zooming consistently outperform the classical Zooming algorithm across all scenarios. This provides strong empirical support for the effectiveness of our proposed quantum approaches. The advantage holds under both types of noise considered, highlighting the benefits of integrating quantum techniques into the Lipschitz bandit framework.

A noteworthy observation is that Q-LAE often achieves comparable or slightly better performance when compared to Q-Zooming. Although Q-Zooming possesses a theoretically superior regret bound due to its favorable polylogarithmic factor, Q-LAE's elimination-based structure offers distinct practical advantages. This practical benefit allows Q-LAE to marginally outperform Q-Zooming over time. Intuitively, Q-LAE achieves this by progressively eliminating low-reward regions, enabling it to become more focused and improve its performance in later stages. However, in the early stages, with most arms still active, Q-LAE explores broadly and may include suboptimal regions, potentially making it slightly less efficient than Q-Zooming initially. Nonetheless, as learning progresses, Q-LAE rapidly converges, ultimately delivering superior overall performance.

\section{Discussion}

In this work, we introduced the first two quantum Lipschitz bandit algorithms, named Q-LAE and Q-Zooming, under the non-linear reward functions and arbitrary continuous arm metric space. We provided a detailed theoretical analysis to illustrate that both of our efficient algorithms can achieve an improved regret bound of order $\tilde O(T^{d_z/(d_z+1)})$ when the noise is bounded or has finite variance. The superiority of our proposed methods over state-of-the-art approaches is validated under comprehensive experiments. 

\paragraph{Comparison of Q-LAE and Q-Zooming: }

While Q-Zooming achieves a better regret bound when considering logarithmic terms, its empirical performance does not consistently outperform that of Q-LAE. As demonstrated in Section 6, Q-LAE can surpass Q-Zooming in certain scenarios due to its elimination-based strategy, which efficiently prunes low-reward regions and concentrates exploration on more promising areas. This targeted approach often leads to faster convergence, especially when the reward function contains large suboptimal regions. Moreover, Q-LAE is the first elimination-based Lipschitz bandit algorithm to adopt the consistent definition of the zooming dimension (see Remark 4.1), leading to a novel and distinctive regret analysis. We believe this offers a valuable contribution to the Lipschitz bandit literature. It is also worth noting that zooming-based algorithms, both classical and quantum, face scalability challenges in high-dimensional settings, whereas Q-LAE does not suffer from this limitation. Therefore, we view Q-LAE as a complementary alternative to Q-Zooming, offering unique strengths in both theoretical formulation and practical effectiveness.

\paragraph{Limitations:} A limitation of our work is the absence of a theoretical lower bound, leaving the optimality of our achieved regret bound uncertain. However, as lower bounds are also unresolved for simpler cases like quantum multi-armed bandit, this would remain a challenging future work.

\section*{Acknowledgments}

This work was supported in part by the National Science Foundation under grants DMS-2152289 and DMS-2134107. We would also like to express sincere gratitude to Dr. Zhongxiang Dai for his invaluable guidance throughout this work.

\bibliography{aaai2026}

\newpage
\appendix
\onecolumn

\section{Analysis of Q-LAE}\label{app:q-lae}

In our Q-LAE algorithm, we construct the maximal packing of the active region \(\mathcal{C}_m\) and use these points to play and estimate rewards. This approach is justified because the selected maximal packing also functions as a covering, ensuring that these selected points effectively represent the entire remaining region. We formally introduce and prove this mathematical fact as a proposition.

\vspace{1mm}

\begin{proposition}[Maximal $\epsilon$-packing implies $\epsilon$-covering]
Let \((X, \mathcal{D})\) be a metric space with a subset \(S \subseteq X\). If \(\{ x_1, x_2, ..., x_n \} \subseteq S\) is a maximal \(\epsilon\)-packing, then:
\[
S \subseteq \bigcup_{i=1}^n \mathcal{B}(x_i, \epsilon),
\]
where \(\mathcal{B}(x_i, \epsilon) \) denotes the closed ball of radius \(\epsilon\) centered at \(x_i\).    
\end{proposition}

\begin{proof}
Suppose there exists a point \(x \in S\) such that \(x \notin \bigcup_{i=1}^n \mathcal{B}(x_i, \epsilon) \), implying $x$ does not belong to \( \mathcal{B}(x_i, \epsilon) \) for any \(i =1,2, \dots, n\). This implies that
\[
\mathcal{D}(x,x_i) > \epsilon \quad \text{for every } i =1,2, \dots, n.
\]
Hence, \(x\) is at least a distance \(\epsilon\) away from all the centers \(x_1, x_2, \dots, x_n\).

This indicates that \(x\) can be added to the packing set \(\{x_1, x_2, \dots, x_n\}\) without violating the packing condition, contradicting the maximality of the \(\epsilon\)-packing.
Therefore, no such \(x\) exists, and we conclude that
\[
S \subseteq \bigcup_{i=1}^n \mathcal{B}(x_i, \epsilon).
\]
\end{proof}

\subsection{Analysis of Theorem~\ref{thm:blin}}\label{app:blin}

In this section, we provide a comprehensive analysis of Theorem~\ref{thm:blin}. 
Before proceeding, we define the \textit{zooming constant} \(C_z\), a concept closely associated with the zooming dimension. It is defined as the smallest multiplier \(\alpha > 0\) that satisfies the following inequality for the zooming dimension \(d_z\):
\[
C_z = \min \{\alpha > 0 : N_z(r) \leq \alpha r^{-d_z}, \, \forall r \in (0, 1] \}.
\]
The zooming constant \(C_z\) will be used throughout the analysis in Appendices~\ref{app:q-lae} and \ref{app:b}.
We now begin with the definition of a clean event.

\begin{definition}
A clean event for the Q-LAE algorithm is defined as the condition where, for every stage \(m\) and all \(x \in \mathcal{A}_m\), the following inequality holds:
\[
|\hat{\mu}_m(x) - \mu(x)| \leq \epsilon_m.
\]
\end{definition}

\vspace{1mm}
\begin{remark}\textit{
    As described in Section~\ref{subsec:Q-LAE_regret}, the probability of a clean event can be established as at least \(1 - \delta\) by leveraging the QMC algorithm.}
\end{remark}

\vspace{2mm}
We present a lemma which, under the assumption that the clean event holds, bounds the optimality gap of each active arm by the confidence radius from the previous stage. This ensures that the remaining arm regions correspond to high-reward regions. Additionally, in the proof of this lemma, we demonstrate that the optimal arm is never eliminated during the Q-LAE algorithm when the clean event holds.

\begin{lemma}\label{lem:blin}
    Under the clean event, for all $m$ and any arm $x \in \mathcal{A}_m$, $\Delta_x$ is bounded by
    $$\Delta_x \leq 7 \epsilon_{m-1}.$$
\end{lemma}

\begin{proof}

We first show that the optimal arm $x^* = \argmax \mu(x)$ is not eliminated after stage $m$. More precisely, we prove that $x^* \in \mathcal{C}_{m+1}$. 

At the beginning of stage $m$, there exists $x^{**} \in \mathcal{A}_m$ such that $x^* \in \mathcal{B}(x^{**}, \epsilon_m)$. For any arm $x \in \mathcal{A}_m$, we have
\begin{align*}
      \hat{\mu}_m(x) - \hat{\mu}_m(x^{**}) & \leq \mu(x)  + \epsilon_m - \mu(x^{**})  + \epsilon_m \\
     & \leq \mu(x)  + \epsilon_m  -  \mu(x^*)+ \epsilon_m  + \epsilon_m\\
     & =  \mu(x) - \mu(x^*)+ 3\epsilon_m \\
     & \leq  3\epsilon_m,
\end{align*}
where the first inequality follows from the definition of the clean event, and the second inequality holds from the Lipschitzness of $\mu$.
Thus, we have $\hat{\mu}_m(x^{**}) \geq \hat{\mu}_m(x) - 3\epsilon_m$ for all $x \in \mathcal{A}_m$. This ensures that $x^{**}$ is not eliminated, meaning $x^{**} \in \mathcal{A}_m^+$. Since $x^* \in \mathcal{B}(x^{**}, \epsilon_m)$, it follows that $x^* \in \mathcal{C}_{m+1}$.

We now proceed with the main proof, using the fact that the optimal arm \(x^*\) is never removed from the active region during the algorithm.
For any arm $x \in \mathcal{A}_m$, we know that $x \in \mathcal{C}_m$. Thus, there exists some $x_0 \in \mathcal{A}_{m-1}^+ \subseteq \mathcal{A}_{m-1}$ such that $x \in \mathcal{B}(x_0, \epsilon_{m-1})$. Similarly, since $x^* \in \mathcal{C}_{m}$, there exists some $x^{**} \in \mathcal{A}_{m-1}^+$ such that $x^* \in \mathcal{B}(x^{**}, \epsilon_{m-1})$.

Therefore, for any arm $x \in \mathcal{A}_m$, it holds that
\begin{align*}
    \Delta_x = \mu^* - \mu(x) & \leq \mu(x^{**}) + \epsilon_{m-1} - \mu(x_0) + \epsilon_{m-1} \\
    & \leq \hat{\mu}_{m-1}(x^{**}) + \epsilon_{m-1}+ \epsilon_{m-1} - \hat{\mu}_{m-1}(x_0) + \epsilon_{m-1}+ \epsilon_{m-1} \\
    & \leq \max_{x \in \mathcal{A}_{m-1}} \hat{\mu}_{m-1}(x) - \hat{\mu}_{m-1}(x_0)+ 4\epsilon_{m-1} \\
    & \leq 3\epsilon_{m-1} + 4\epsilon_{m-1} = 7\epsilon_{m-1}.
\end{align*}

The first and second inequalities hold due to the Lipschitzness of $\mu$ and the definition of the clean event, respectively. The final inequality follows from the fact that \(x_0 \in \mathcal{A}_{m-1}^+\), which ensures that it was not eliminated by the elimination rule. This completes the proof.

\end{proof}

\vspace{1mm}
We now present the proof of Theorem~\ref{thm:blin}.

\thmblin*
\begin{proof}


For \( i \in \mathbb{N} \), let \( r = 2^{-i} \) and define \( Y_i := X_r\).
Next, consider the set of arms that are active at stage $m$ and also belong to $Y_i$.
\[
Z_{i,m} := Y_i \cap \mathcal{A}_m.
\]

For any two distinct arms \( x, y \in Z_{i,m} \), the distance between them is lower bounded by
\[
D(x, y) \geq \epsilon_m = \frac{1}{2} \epsilon_{m-1} \geq \frac{1}{14} \Delta_x \geq \frac{1}{14} r.
\]
This inequality holds because \(\mathcal{A}_m\) forms an \(\epsilon_m\)-packing and by Lemma~\ref{lem:blin}. 
As a result, when the set \( Y_i \) is covered by subsets with diameter \( r/14 \), arms \( x \) and \( y \) cannot be included in the same subset.
Consequently, we have:
\[
|Z_{i,m}| \leq N_z(r) \leq C_z r^{-d_z},
\]
where \( d_z \) is the zooming dimension and \( C_z \) is the corresponding constant.


Now, consider the regret contribution from arms in \( Z_{i,m} \):
\begin{align*}
    R_{i,m} = \sum_{x \in Z_{i,m}} n_m \Delta_x & \leq \sum_{x \in Z_{i,m}} \frac{C_1}{\epsilon_m} \cdot O(\log T) \cdot 2r \\
    & = |Z_{i,m}| \cdot \frac{C_1}{\epsilon_m} \cdot O(\log T) \cdot 2r \\
    & \leq C_z r^{-d_z} \cdot \frac{C_1}{\epsilon_m} \cdot O(\log T) \cdot 2r.
\end{align*}
Choose \( \alpha > 0 \) and analyze the arms based on whether \( \Delta_x \leq \alpha \) or \( \Delta_x > \alpha \). Let \( T_m \) denote the number of rounds in stage \( m \). Then, the regret at stage \( m \) can be bounded as:
\begin{align*}
    R_m &\leq \alpha \cdot T_m + \sum_{i: r = 2^{-i} > \alpha} R_{i,m} \\
& \leq  \alpha \cdot T_m + \sum_{i: r = 2^{-i} > \alpha} \frac{r}{\epsilon_m} \Theta(\log T) \cdot  C_z r^{-d_z} \\
& \leq  \alpha \cdot T_m + \sum_{i: r = 2^{-i} > \alpha} \Theta(\log T) \cdot  C_z r^{-d_z} \\ 
& \leq \alpha \cdot T_m + O(C_z \log T) \cdot \frac{1}{\alpha^{d_z}}
\end{align*}
Now, choose $\alpha = \left( \frac{C_z \log T}{T_m} \right)^{\frac{1}{d_z+1}}$. Substituting this value yields
$$R_m \leq C_z^{\frac{1}{d_z+1}} \cdot T_m^{\frac{d_z}{d_z+1}}  (\log T)^{\frac{1}{d_z+1}}.$$
Since $f(x) = x^{\frac{d_z}{d_z+1}}$ is a concave function on $[0,\infty)$, we can apply Jensen's inequality to obtain
$$T_1^{\frac{d_z}{d_z+1}} + ... + T_B^{\frac{d_z}{d_z+1}} \leq  B^{\frac{1}{d_z+1}} \cdot T^{\frac{d_z}{d_z+1}},$$
for any $B$.
Then summing the regret across all stages $1 \leq m \leq B$, we get
\begin{align*}
\sum_{m=1}^B R_m  & \leq C_z^{\frac{1}{d_z+1}} \sum_{m=1}^B T_m^{\frac{d_z}{d_z+1}}  (\log T)^{\frac{1}{d_z+1}} \\
& \leq C_z^{\frac{1}{d_z+1}} B^{\frac{1}{d_z+1}}  \cdot T^{\frac{d_z}{d_z+1}}  (\log T)^{\frac{1}{d_z+1}}. 
\end{align*}
Therefore, we can bound the total regret as
$$R(T)  \leq C_z^{\frac{1}{d_z+1}} B^{\frac{1}{d_z+1}}  \cdot T^{\frac{d_z}{d_z+1}}  (\log T)^{\frac{1}{d_z+1}} + \sum_{m > B} R_m$$
Using Lemma~\ref{lem:blin}, we further obtain
$$ \leq C_z^{\frac{1}{d_z+1}} B^{\frac{1}{d_z+1}}  \cdot T^{\frac{d_z}{d_z+1}}  (\log T)^{\frac{1}{d_z+1}} + 7 \cdot 2^{-B} \cdot T$$
and simplifying this yields
$$ = O \left( T^{\frac{d_z}{d_z+1}}  (\log T)^{\frac{2}{d_z+1}} \right), $$
by choosing $B= \frac{\log\frac{T}{\log T}}{d_z +1 }$ to balance the two terms optimally.
\end{proof}

\vspace{2mm}
\begin{remark}
In this proof, we use a slightly different definition of the zooming number compared to the one defined in Section~\ref{sec:prelim}. The zooming number \(N_z(r)\) was defined as the minimal number of radius-\(r/3\) balls required to cover the set of near-optimal arms \(X_r = \{x \in X : r \leq \Delta_x < 2r\}\). However, in the proof of this theorem, we utilize the zooming number as the minimal number of radius-\(r/14\) balls required to cover the set \(X_r\). 

It is worth noting that using different values for the radius is a common practice. For instance, \citet{slivkins2019introduction} uses \(r/3\), \citet{kleinberg2019bandits, kang2023robust} use \(r/16\), and \citet{feng2022lipschitz} uses \(r/2\). 
Here, we explain that using different radius values in the definition of the zooming number does not affect the equivalence of the zooming dimension.

Let \(N_1(r)\) denote the number of radius-\(r\) balls required to cover \(X_r\), and let \(N_2(r)\) denote the number of radius-\(r/c\) balls required to cover \(X_r\). Suppose the zooming dimensions computed using \(N_1(r)\) and \(N_2(r)\) are \(d_1\) and \(d_2\), respectively.
By Assouad’s embedding theorem~\cite{assouad1983plongements}, any compact doubling metric space \((X, \mathcal{D})\) can be embedded into a Euclidean space with some metric. 
Therefore, without loss of generality, we can restrict our analysis to a Euclidean space \( \mathbb{R}^k\) equipped with any metric. For such a Euclidean space \( \mathbb{R}^k\), we can apply Lemma 5.7 from \cite{wainwright_2019}, which establishes the relationship between \(N_1(r)\) and \(N_2(r)\) as:
\begin{equation}\label{eq:app_zooming_number}
    N_1(r) \leq N_2(r) \leq (1 + 2c)^k N_1(r).
\end{equation}

By the definition of $d_1$, there exists some \(\alpha_1 > 0\) such that \(N_1(r) \leq \alpha_1 r^{-d_1}\). Therefore, we have 
\[
N_2(r) \leq (1 + 2c)^k N_1(r) \leq \alpha_1 (1 + 2c)^k r^{-d_1}
\]
by Equation \eqref{eq:app_zooming_number}.
Since \(\alpha_1\), \(c\), and \(k\) are constants independent of \(r\), and based on the definition of the zooming dimension, we conclude that \(d_2 \leq d_1\). 

Conversely, using the fact that \(N_1(r) \leq N_2(r)\), we also deduce that \(d_1 \leq d_2\). Therefore, we conclude that \(d_1 = d_2\), meaning the zooming dimension remains the same regardless of the radius value used in the definition of the zooming number.

\end{remark}

\vspace{2mm}

\subsection{Analysis of Theorem~\ref{thm:blin_bounded_variance}}\label{app:blin_bounded_variance}

When the reward contains noise with bounded variance, the Q-LAE algorithm substitutes $\text{QMC}_1$ with $\text{QMC}_2$ as described in Lemma~\ref{lem:qmc}. Additionally, the number of rounds each arm is played during the stage is adjusted to match the number of queries required by $\text{QMC}_2$. Please refer to Algorithm~\ref{alg:Q-LAE2} for the Q-LAE algorithm designed for the bounded variance noise setting.

\begin{algorithm}[ht]
\caption{Q-LAE Algorithm under bounded variance noise}\label{alg:Q-LAE2}
\textbf{Input:} time horizon $T$, fail probability $\delta$ \\
\textbf{Initialization:} $\mathcal{A}_1 \leftarrow \text{maximal-} \frac12 \text{ packing of } X$, \\
 $\mathcal{C}_1 \leftarrow X$, $\epsilon_m = 2^{-m}$ for all $m$
\begin{algorithmic}[1]
\FOR{stage $m = 1,2,...,B$}
\STATE {\color{red} $n_m \leftarrow \frac{C_2 \sigma}{\epsilon_m} \log^{3/2}_2\left(\frac{8\sigma}{\epsilon_m}\right) \log_2\left(\log_2\frac{8\sigma}{\epsilon_m}\right) \log\left( \frac{T}{\delta}\right)$}
\FOR{each $x \in \mathcal{A}_m$}
\STATE Play $x$ for the next $n_m$ rounds and obtain $\hat{\mu}_m(x)$ by running the $\text{QMC}_2(\mathcal{O}_{x},\epsilon_m,\delta/T)$ algorithm.
\ENDFOR
\LineComment{selective elimination}
\STATE $\hat{\mu}_{max} \leftarrow \max_{x \in \mathcal{A}_m}\hat{\mu}_m(x)$
\STATE For each $x \in \mathcal{A}_m$, eliminate $x$ if $\hat{\mu}_m(x) < \hat{\mu}_{max} - 3\epsilon_m$. Let $\mathcal{A}_m^+$ denote the set of points not eliminated.
\LineComment{progressive refinement}
\STATE $\mathcal{C}_{m+1} \leftarrow \bigcup_{x \in \mathcal{A}_m^+} \mathcal{B}(x, \epsilon_m)$
\STATE Find a maximal $\epsilon_{m+1}$-packing of $\mathcal{C}_{m+1}$ and define it as $\mathcal{A}_{m+1}$.
\ENDFOR
\end{algorithmic}
\end{algorithm}

\thmblinvariance*
\begin{proof}
    The proof of Theorem~\ref{thm:blin_bounded_variance} follows the same argument of the proof of Theorem~\ref{thm:blin}, using the same notation. The number of rounds $n_m$ each arm is played in stage $m$ changes from $\frac{C_1}{\epsilon_m} \log \left( \frac{T}{\delta} \right) $ to $\frac{C_2 \sigma}{\epsilon_m} \log^{3/2}_2\left(\frac{8\sigma}{\epsilon_m}\right) \log_2\left(\log_2\frac{8\sigma}{\epsilon_m}\right) \log \left( \frac{T}{\delta} \right)$.
    
    For stage $m$ and $x \in Z_{i,m}$, it holds that
    \begin{align*}
    n_m  \Delta_x & \leq \frac{C_2 \sigma}{\epsilon_m} \log^{3/2}_2\left(\frac{8\sigma}{\epsilon_m}\right) \log_2\left(\log_2\frac{8\sigma}{\epsilon_m}\right) \cdot O(\log T) \cdot 2r \\
    & = \frac{C_2 \sigma}{\epsilon_m} \cdot m^{\frac32} \log m \cdot O(\log T) \cdot 2r.
    \end{align*}
    
    Thus, the regret from arms in \( Z_{i,m} \) is
\begin{align*}
    R_{i,m} = \sum_{x \in Z_{i,m}} n_m \Delta_x & \leq \sum_{x \in Z_{i,m}} \frac{C_2 \sigma}{\epsilon_m} \cdot m^{\frac32} \log m \cdot O(\log T) \cdot 2r \\
    & = |Z_{i,m}| \cdot \frac{C_2 \sigma}{\epsilon_m} \cdot m^{\frac32} \log m \cdot O(\log T) \cdot 2r \\
    & \leq C_z r^{-d_z} \cdot \frac{C_2 \sigma}{\epsilon_m} \cdot m^{\frac32} \log m \cdot O(\log T) \cdot 2r.
\end{align*}

Similar to the proof of Theorem~\ref{thm:blin}, we can bound $R_m$ as
\begin{align*}
    R_m &\leq \alpha \cdot T_m + \sum_{i: r = 2^{-i} > \alpha} R_{i,m} \\
& \leq  \alpha \cdot T_m + \sum_{i: r = 2^{-i} > \alpha} \frac{r}{\epsilon_m} \cdot m^{\frac32} \log m \cdot \Theta(\log T) \cdot  C_z r^{-d_z} \\
& \leq  \alpha \cdot T_m + m^{\frac32} \log m  \sum_{i: r = 2^{-i} > \alpha} \Theta(\log T) \cdot  C_z r^{-d_z} \\ 
& \leq \alpha \cdot T_m + m^{\frac32} \log m  \cdot O(C_z \log T) \cdot \frac{1}{\alpha^{d_z}}.
\end{align*}

Then, choosing $\alpha = \left( \frac{ m^{\frac32} \log m  \cdot C_z \log T}{T_m} \right)^{\frac{1}{d_z+1}}$ and substituting gives
\begin{align*}
    R_m & \leq C_z^{\frac{1}{d_z+1}} \cdot T_m^{\frac{d_z}{d_z+1}}  (\log T)^{\frac{1}{d_z+1}} \cdot (m^{\frac32} \log m)^{\frac{1}{d_z+1}}. 
\end{align*}

Summing the regret over all stages $1 \leq m \leq B$ and applying Jensen's inequality yields
\vspace{2mm}
\begin{align*}
    \sum_{m=1}^B R_m  & \leq C_z^{\frac{1}{d_z+1}} \cdot (B^{\frac32} \log B)^{\frac{1}{d_z+1}} \sum_{m=1}^B T_m^{\frac{d_z}{d_z+1}}  (\log T)^{\frac{1}{d_z+1}} \\
& \leq C_z^{\frac{1}{d_z+1}} \cdot B^{\frac52 \frac{1}{d_z+1}} (\log B)^{\frac{1}{d_z+1}} \cdot T^{\frac{d_z}{d_z+1}}  (\log T)^{\frac{1}{d_z+1}} .
\end{align*}
Finally, similar to the proof of Theorem~\ref{thm:blin}, we choose an optimal $B$ and obtain
$$R(T) = O \left( T^{\frac{d_z}{d_z+1}}  (\log T)^{\frac72 \frac{1}{d_z+1}} (\log\log T)^{\frac{1}{d_z+1}} \right).$$

\end{proof}

\section{Analysis of Q-Zooming}\label{app:b}

\subsection{Analysis of Theorem~\ref{thm:q-zooming}}\label{app:zooming}

In this section, we conduct a detailed analysis of Theorem~\ref{thm:q-zooming}, starting with defining a clean event.

\vspace{1mm}
\begin{definition}
A clean event for the Q-Zooming algorithm is defined as the condition where, all \( 1 \leq s \leq m \) and every arm \( x \), the following inequality holds:
\[
|\hat{\mu}_s(x) - \mu(x)| \leq \epsilon_s(x).
\]
\end{definition}

\vspace{0.5mm}
\begin{remark}\textit{
In the proof of Theorem~\ref{thm:q-zooming} below, we demonstrate that the QMC algorithm guarantees the probability of a clean event to be at least \(1 - \delta\).
}
\end{remark}
\vspace{1mm}

We provide a lemma that bounds the optimality gap of any active arm using the confidence radius from the previous stage, assuming the clean event holds. Notably, the value of confidence radius \(\epsilon_s\) decreases at each stage only if the corresponding arm is selected during that stage. This lemma indicates that, in later stages, only arms with small optimality gaps \(\Delta_x\)-- i.e., those with rewards close to the optimal arm's reward -- continue to be selected. This result will play a crucial role in bounding the total regret in Theorem~\ref{thm:q-zooming}.

\vspace{1mm}
\begin{lemma}\label{lem:q1}
    Under the clean event, for any arm  $x$ and each stage $s$, $\Delta_x$ is bounded by $$\Delta_x \leq 3 \epsilon_{s-1}(x).$$
\end{lemma}

\begin{proof}

Let us first consider the arm \(x_s\), which is selected in stage \(s\). By the activation rule, in stage \(s\), the optimal arm \(x^* = \arg\max \mu(x)\) must be covered by the confidence ball of some active arm \(x^{**}\), i.e., $x^* \in \mathcal{B}(x^{**}, \epsilon_{s-1}(x^{**}))$.

Using the Lipschitz continuity of \(\mu\) and $x^* \in \mathcal{B}(x^{**}, \epsilon_{s-1}(x^{**}))$, we can bound the  \(\Delta_{x_s}\) as follows:
\[
\Delta_{x_s} = \mu^* - \mu(x_s) \leq \mu(x^{**}) + \epsilon_{s-1}(x^{**}) - \mu(x_s).
\]
Next, using the clean event condition, this can be further bounded as
\begin{align*}
\Delta_{x_s} & \leq \mu(x^{**}) + \epsilon_{s-1}(x^{**}) - \mu(x_s) \\
& \leq \hat{\mu}_{s-1}(x^{**}) + 2 \epsilon_{s-1}(x^{**}) -  \hat{\mu}_{s-1}(x_s) +  \epsilon_{s-1}(x_s).
\end{align*}
Furthermore, since $x_s$ was selected in stage $s$, this implies that 
$$\hat{\mu}_{s-1}(x^{**}) + 2 \epsilon_{s-1}(x^{**}) \leq \hat{\mu}_{s-1}(x_s) + 2 \epsilon_{s-1}(x_s)$$
by the selection rule.
Substituting this inequality into the previous bound, we obtain
\begin{align*}
    \Delta_{x_s} & \leq \hat{\mu}_{s-1}(x_s) + 2\epsilon_{s-1}(x_s) - \hat{\mu}_{s-1}(x_s) + \epsilon_{s-1}(x_s) \\
    & = 3\epsilon_{s-1}(x_s).
\end{align*}
This completes the proof for \(x_s\), the arm selected in each stage.

For any arm \(x\) that was not previously selected, \(\epsilon_{s-1}(x) = 1\), and thus the lemma holds trivially. Additionally, for any arm \(x\) that was selected in a prior stage, since \(\epsilon_s(x)\) remains unchanged after the stage in which it was last chosen, the inequality $\Delta_x \leq 3 \epsilon_{s-1}(x)$ continues to hold.

\end{proof}

Using the lemma, we now derive a high-probability upper bound on the regret of the Q-Zooming algorithm.

\thmzooming*
\begin{proof}
Following each stage, Lemma~\ref{lem:qmc} guarantees that QMC, with a sufficient number of queries \( N_s \), provides an estimate \(\hat{\mu}_s(x_s)\) satisfying \(|\hat{\mu}_s(x_s) - \mu(x_s)| \leq \epsilon_s(x_s)\) with probability at least \(1 - \frac{\delta}{m}\), where \( m \) is the total number of stages. By applying the union bound over all stages $1 \leq s \leq m$, it follows that \(|\hat{\mu}_s(x_s) - \mu(x_s)| \leq \epsilon_s(x_s)\) holds for all \( 1 \leq s \leq m \) with probability at least \(1 - \delta\).

Additionally, since \(\epsilon_s(x)\) is set to 1 for arms that have never been played and both \(\epsilon_s(x)\) and \(\hat{\mu}_s(x)\) remain unchanged for arms \( x \) that are not selected during stage $s$, this result can be extended. Specifically, \(|\hat{\mu}_s(x) - \mu(x)| \leq \epsilon_s(x)\) holds for all \( 1 \leq s \leq m \) and every arm \( x \), with probability at least \(1 - \delta\). We define this condition as a clean event and assume it holds throughout the subsequent analysis.

For \( r > 0 \), define the set of near-optimal arms with an optimality gap between \( r \) and \( 2r \) as
\[
X_r := \left\{ x \in X : r \leq \Delta_x < 2r \right\}.
\]
For \( i \in \mathbb{N} \), let \( r = 2^{-i} \) and define \( Y_i := X_r\).
The regret contribution from the arms in \( Y_i \) can be expressed as
\[
R_i(T) = \sum_{x \in Y_i} M_x(T) \, \Delta_x,
\]
where \( M_x(T) \) denotes the number of times arm \( x \) is pulled up to time \( T \). Additionally, let \( s(x) \) represent the last stage in which arm \( x \) was played before time \( T \). We claim that
\[
M_x(T) \leq 2 N_{s(x)}.
\]
Let \( k(x) \) be the number of stages where arm \( x \) was played. Then, the total number of pulls for arm \( x \) is given by
\[
M_x(T) = \sum_{i=1}^{k(x)} 2^i C_1 \log \left( \frac{m}{\delta} \right).
\]
Since \( \sum_{i=1}^{k(x)} 2^i = 2^{k(x)+1} - 2 \), it follows that
\[
M_x(T) \leq 2^{k(x)+1} C_1 \log \left( \frac{m}{\delta} \right) = 2 N_{s(x)}.
\]
Substituting this bound into \( R_i(T) \), we get
\[
R_i(T) \leq \sum_{x \in Y_i} \frac{2 C_1}{\epsilon_{s(x)}(x)} \log \left( \frac{m}{\delta} \right) \Delta_x.
\]
Using Lemma~\ref{lem:q1}, we can bound \(\Delta_x\) as follows: 
\begin{equation}\label{eq:3}
    \Delta_x \leq 3 \epsilon_{s(x)-1}(x) \leq 6 \epsilon_{s(x)}(x).
\end{equation}

In addition, we claim that the number of elements in \( Y_i \) is bounded by \( N_z(r) \). Under the clean event, for any pair of active arms \(x, y \in Y_i\), we have \(\mathcal{D}(x, y) > \frac{r}{3}\). This is because, without loss of generality, suppose \(y\) was activated after \(x\), and let \(s\) denote the stage in which \(y\) was activated. This means, at this stage $s$, \(y\) was not covered by any confidence ball of active arms, including \(x\). Therefore:
\[
\mathcal{D}(x, y) > \epsilon_{s-1}(x),
\]
and by Lemma~\ref{lem:q1}, we can further deduce
\[
\mathcal{D}(x, y) > \epsilon_{s-1}(x) \geq \frac{1}{3} \Delta_x \geq \frac{1}{3} r.
\]

As a result, when the set \(Y_i\) is covered by sets of diameter \(r/3\), arms \(x\) and \(y\) cannot belong to the same set. This implies that the number of elements in \( Y_i \) is bounded by \( N_z(r) \). Using this and Equation~\eqref{eq:3}, we can derive the following bound for the regret contribution from arms in \( Y_i \):

\begin{align*}
    R_i(T) & \leq \sum_{x \in Y_i} 12 C_1 \log \left( \frac{m}{\delta} \right) \\
    & \leq 12 C_1 \log \left( \frac{m}{\delta} \right) \cdot N_{z}(r).
\end{align*}
Letting \(\delta \geq \frac{1}{T}\) and noting \(m \leq T\), this simplifies to:
\[
R_i(T) \leq O(\log T) \cdot N_{z}(r).
\]
Next, consider a fixed \( \alpha > 0 \) and analyze the arms based on whether \( \Delta_x \leq \alpha \) or \( \Delta_x > \alpha \). The total regret $R(T)$ can then be written as:
\[
R(T) = \alpha T + \sum_{i : r = 2^{-i} > \alpha} R_i(T).
\]
By substituting the bound for \( R_i(T) \), we obtain:
\begin{align*}
R(T) & \leq \alpha T + \sum_{i : r = 2^{-i} > \alpha} O(\log T) \cdot C_z  r^{-d_z} \\
& \leq \alpha T + O(C_z \log T) \cdot \frac{1}{\alpha^{d_z}},
\end{align*}
where the first inequality follows from the definition of the zooming dimension. 
Choosing \( \alpha = \left( \frac{C_z \log T}{T} \right)^{\frac{1}{d_z+1}} \), we conclude:
\[
R(T) = O \left( T^{\frac{d_z}{d_z+1}} \cdot (C_z \log T)^{\frac{1}{d_z+1}} \right) = O \left( T^{\frac{d_z}{d_z+1}} \cdot (\log T)^{\frac{1}{d_z+1}} \right).
\]
\end{proof}

\vspace{1mm}

\subsection{Analysis of Theorem~\ref{thm:zooming_bounded_variance}}\label{app:zooming_bounded_variance}

In the case where the reward contains noise with bounded variance, the Q-Zooming algorithm replaces \(\text{QMC}_1\) with \(\text{QMC}_2\). Furthermore, the number of rounds played in each stage is modified to align with the query requirements of \(\text{QMC}_2\), as outlined in Lemma~\ref{lem:qmc}. See Algorithm~\ref{alg:q-zooming2} for the implementation of the Q-Zooming algorithm designed for the bounded variance noise setting.

\begin{algorithm}[ht]
\caption{Q-Zooming Algorithm}\label{alg:q-zooming2}
\textbf{Input:} time horizon $T$, fail probability $\delta$ \\
\textbf{Initialization:} active arms set $S \leftarrow \emptyset$, confidence radius $\epsilon_0(\cdot)=1$
\begin{algorithmic}[1]
\FOR{stage $s = 1,2,...$}
\LineComment{activation rule}
\IF{there exists an arm $y$ that is not covered by the confidence balls of active arms}
\STATE add any such arm $y$ to the active set: $S \leftarrow S \cup \{ y \}$
\ENDIF
\LineComment{selection rule}
\STATE $x_s \leftarrow \argmax_{x \in S} \hat{\mu}_{s-1}(x)+ 2 \epsilon_{s-1}(x) $
\STATE $\epsilon_s(x) \leftarrow \begin{cases}
    \epsilon_{s-1}(x) /2  \quad \text{if } x = x_s, \\
    \epsilon_{s-1}(x)  \quad \text{if } x \neq x_s.
\end{cases}$
\STATE  {\color{red} $N_{s} \leftarrow \frac{C_2 \sigma}{\epsilon_s(x_s)} \log^{3/2}_2\left(\frac{8\sigma}{\epsilon_s(x_s)}\right) \log_2\left(\log_2\frac{8\sigma}{\epsilon_s(x_s)}\right) \log\left( \frac{m}{\delta}\right)$}
\STATE If $\sum_{k=1}^{s} N_s > T$, terminate the algorithm.
\STATE Play $x_s$ for the next $N_s$ rounds and obtain $\hat{\mu}_s(x_s)$ by running the $\text{QMC}_2(\mathcal{O}_{x_s},\epsilon_s(x_s),\delta/m)$ algorithm.\\
\quad \textit{Note: For all other arms \( x \neq x_s \), retain the reward estimates \( \hat{\mu}_{s-1}(x) \) from stage \( s-1 \).}
\ENDFOR
\end{algorithmic}
\end{algorithm}

\thmzoomingvariance*
\begin{proof}
    The proof of Theorem~\ref{thm:zooming_bounded_variance} proceeds similarly to the proof of Theorem~\ref{thm:q-zooming} and employs the same notation. The number of rounds $N_s$ that arm $x_s$ is played in stage $s$ changes from $\frac{C_1}{\epsilon_s(x_s)} \log \left( \frac{m}{\delta} \right) $ to $\frac{C_2 \sigma}{\epsilon_s(x_s)} \log^{3/2}_2\left(\frac{8\sigma}{\epsilon_s(x_s)}\right) \log_2\left(\log_2\frac{8\sigma}{\epsilon_s(x_s)}\right) \log \left( \frac{m}{\delta} \right)$.

    Additionally, we can show that \( M_x(T) \leq 2N_{s(x)} \) as follows:
\begin{align*}
    M_x(T) & = \sum_{i=1}^{k(x)} 2^i C_2 \sigma \cdot \log^{3/2}_2\left(\frac{8\sigma}{2^{-i}}\right) \log_2\left(\log_2\frac{8\sigma}{2^{-i}}\right) \log \left( \frac{m}{\delta} \right) \\
    & \leq  2^{k(x)+1} C_2 \sigma \cdot \log^{3/2}_2\left(\frac{8\sigma}{2^{-k(x)}}\right) \log_2\left(\log_2\frac{8\sigma}{2^{-k(x)}}\right) \log \left( \frac{m}{\delta} \right) = 2 N_{s(x)}.
\end{align*}
Therefore, the regret \( R_i(T) \) can be bounded as
\begin{align*}
    R_i(T) & = \sum_{x \in Y_i} M_x(T) \, \Delta_x \\
    & \leq \sum_{x \in Y_i} 2N_{s(x)} \, \Delta_x \\
    & \leq \sum_{x \in Y_i} \frac{2C_2 \sigma}{\epsilon_{s(x)}(x)} \log^{3/2}_2\left(\frac{8\sigma}{\epsilon_{s(x)}(x)}\right) \log_2\left(\log_2\frac{8\sigma}{\epsilon_{s(x)}(x)}\right) \log \left( \frac{m}{\delta} \right) \cdot 6 \epsilon_{s(x)}(x) \\
    & \leq O\left( \left( \log T \right)^\frac52 \log\log T   \right)\cdot N_{z}(r).
\end{align*}
This result follows from Lemma~\ref{lem:q1} and the fact that \(\epsilon_{s(x)}(x) = 2^{-k(x)}\) and \(k(x) \leq \log_2 T\).

Again, following the proof of Theorem~\ref{thm:q-zooming}, we analyze the arms based on whether their optimality gap exceeds \(\alpha\). By choosing \( \alpha = \left( \frac{C_z (\log T)^\frac52 \log\log T}{T} \right)^{\frac{1}{d_z+1}} \), we can conclude that
\begin{align*}
    R(T) & = O \left( C_z^{\frac{1}{d_z+1}} \cdot T^{\frac{d_z}{d_z+1}}  (\log T)^{\frac52 \frac{1}{d_z+1}} (\log\log T)^{\frac{1}{d_z+1}} \right) \\
    & = O \left( T^{\frac{d_z}{d_z+1}}  (\log T)^{\frac52 \frac{1}{d_z+1}} (\log\log T)^{\frac{1}{d_z+1}} \right).
\end{align*}

\end{proof}

\section{Related Work on Lipschitz Bandits}\label{app:related-lipschitz}
Lipschitz bandits, along with other notable extensions of the multi-armed bandit framework such as dueling bandits~\cite{yue2012k,yi2024biased,verma2024neural}, kernelized bandits~\cite{dai2019bayesian,dai2024batch,chowdhury2017kernelized}, and parametric bandits~\cite{filippi2010parametric,abbasi2011improved,kang2022efficient, kang2025single}, have been extensively explored to tackle increasingly complex problem settings.
Most prior research on the stochastic Lipschitz bandits builds on two primary approaches. The first involves uniformly discretizing the action space into a fixed grid, enabling the application of any standard multi-armed bandit algorithms~\cite{kleinberg2004nearly,magureanu2014lipschitz}. The second approach emphasizes dynamic discretization, where more arms are sampled in promising regions of the action space. In this approach, classical methods such as Upper Confidence Bound (UCB)~\cite{bubeck2008online, kleinberg2019bandits,lu2019optimal}, Thompson Sampling (TS)~\cite{kang2024online}, and elimination method~\cite{feng2022lipschitz} can be utilized to balance the exploration-exploitation tradeoff.

Furthermore, \citet{podimata2021adaptive} extended the adaptive discretization framework to address adversarial rewards and derived instance-dependent regret bounds. Recently, \citet{kang2023robust} explored an intermediate setting by introducing Lipschitz bandits with adversarial corruptions, bridging the gap between stochastic and adversarial scenarios. Beyond these works, various extensions of Lipschitz bandits have been thoroughly explored, including contextual Lipschitz bandits~\cite{slivkins2011contextual}, non-stationary Lipschitz bandits~\cite{kang2024online}, and taxonomy bandits~\cite{slivkins2011multi}. Given the broad applicability of Lipschitz bandits in various domains~\cite{slivkins2019introduction}, our work improves existing algorithms both theoretically and empirically by incorporating recent advances in quantum machine learning.

\vspace{1mm}
\section{Upper bound on Zooming Dimension}\label{app:zooming_dimension}

Consider \((X, \mathcal{D})\) as a compact doubling metric space with a covering dimension \(d_c\) and a diameter of at most 1.
As established by Assouad’s embedding theorem~\cite{assouad1983plongements}, any compact doubling metric space \((X, \mathcal{D})\) is embeddable into a Euclidean space with an appropriate metric. Consequently, we can focus our analysis on the space \([0,1]^{d_c}\) under some distance $\mathcal{D}$ without any loss of generality.
Consequently, the reward function \(\mu\) is defined on \([0,1]^{d_c}\), and we denote the optimal arm as \(x^* = \arg\max_{x \in [0,1]^{d_c}} \mu(x)\).

The goal of this section is to demonstrate that the zooming dimension \(d_z\) is often significantly smaller than the covering dimension \(d_c\). This result underscores the fact that the regret bound \(\tilde{O}\left(T^{d_z / (d_z + 1)}\right)\), achieved by our two algorithms, Q-LAE (Section~\ref{sec:Q-LAE}) and Q-Zooming (Section~\ref{sec:q-zooming}), represents a substantial improvement over the optimal regret bound of \(\tilde{O}\left(T^{(d_z + 1) / (d_z + 2)}\right)\) attained by classical Lipschitz bandit algorithms.

We now present a proposition that provides an upper bound on the zooming dimension \(d_z\), along with a rigorous proof.

\begin{proposition}[Relation between zooming and covering dimension] \label{prop:zooming_covering}
Let \( d_z \) be the zooming dimension and \( d_c \) the covering dimension of the metric space. Then, the following holds:
\begin{itemize}
    \item The zooming dimension is bounded as \( 0 \leq d_z \leq d_c \).
    \item Under some mild conditions on \( \mu \), the zooming dimension \( d_z \) can be significantly smaller than \( d_c \). Specifically, if the reward function \( \mu \) satisfies a polynomial growth condition around \( x^* \), that is, if there exist constants \( C > 0 \) and \( \beta > 0 \) such that for all \( x \) in a neighborhood of \( x^* \),
    \[
    \mu(x^*) - \mu(x) \geq C \|x^* - x\|^\beta,
    \]
    then the zooming dimension satisfies:
    \[
    d_z \leq \left( 1 - \frac{1}{\beta} \right) d_c.
    \]
\end{itemize}
\end{proposition}

\begin{proof}
We begin by recalling the definition of the zooming number \( N_z(r) \) from Section~\ref{sec:prelim}. The zooming number \( N_z(r) \) represents the minimal number of radius-\( r/3 \) balls needed to cover the set of near-optimal arms
\[
X_r = \{x \in X : r \leq \Delta_x < 2r\}.
\]
Similarly, we define the covering number \( N_c(r) \) as the minimal number of radius-\( r/3 \) balls required to cover the entire set \( X \).
Since \( X_r \subseteq X \) by definition, it follows directly that
\[
N_z(r) \leq N_c(r).
\]
Next, we recall the formal definitions of the zooming and covering dimensions. The zooming dimension \( d_z \) is defined as
\[
d_z := \min \{ d \geq 0 : \exists \alpha > 0, \, N_z(r) \leq \alpha r^{-d}, \, \forall r \in (0, 1] \}.
\]
Similarly, the covering dimension \( d_c \) is given by:
\[
d_c := \min \{ d \geq 0 : \exists \alpha > 0, \, N_c(r) \leq \alpha r^{-d}, \, \forall r \in (0, 1] \}.
\]
By definition, \( d_z \geq 0 \) holds trivially. Additionally, since we established that \( N_z(r) \leq N_c(r) \) for all \( r \), it follows that:
\[
d_z \leq d_c.
\]
This completes the first proof.

Secondly, we assume that the reward function \( \mu \) satisfies a polynomial growth condition around the optimal arm \( x^* \). That is, there exist constants \( C > 0 \), \( \beta > 0 \), and \( \delta > 0 \) such that for all \( x \) satisfying \( \|x^* - x\| \leq \delta \), we have:
\[
\mu(x^*) - \mu(x) \geq C \|x^* - x\|^\beta.
\]

Then, for any \( r \) satisfying \( 0 < r < C \delta^\beta \), it follows from the polynomial growth condition that:
\begin{align*}
    X_r = \{ x : r \leq \Delta_x = \mu(x^*) - \mu(x) < 2r \} & \subseteq \{ x : C \|x^* - x\|^\beta < 2r  \} \\
    & =  \left\{ x :  \|x^* - x\| < \left( \frac{2r}{C} \right)^\frac{1}{\beta} \right\}.  
\end{align*}

Now, define the set $Y := \left\{ x : \|x^* - x\| < \left(\frac{2r}{C}\right)^{\frac{1}{\beta}} \right\}$.
This represents an Euclidean ball centered at \( x^* \) with radius \( \left(\frac{2r}{C}\right)^{\frac{1}{\beta}} \). To cover \( Y \) with radius-\( r/3 \) balls, we require at most of the order of 
$$\left( \frac{ \left( \frac{2r}{C} \right)^\frac{1}{\beta}}{ \frac{r}{3}} \right)^{d_c} = \left( \frac{3 \times 2^{\frac{1}{\beta}}}{C^\frac{1}{\beta}}\right)^{d_c} \cdot r^{-\left( 1- \frac{1}{\beta} \right) d_c}.$$
Since \( X_r \subseteq Y \), this serves as an upper bound for \( N_z(r) \), leading to:
\[
d_z \leq \left( 1 - \frac{1}{\beta} \right) d_c.
\]
This concludes the second proof.
\end{proof}

This proposition illustrates that the zooming dimension $d_z$ is often strictly smaller than the covering dimension $d_c$. For example, if the reward function $\mu(\cdot)$ is twice continuously differentiable $C^2$-smooth and strongly concave in a neighborhood around $x^*$, then the zooming dimension satisfies $d_z \leq \frac{d_c}{2}$.


\end{document}